\documentclass[a4paper]{article}
\usepackage[utf8]{inputenc}
\usepackage{amsfonts}
\usepackage{amsthm, amsmath, mathtools, amssymb, amscd, mathrsfs}
\usepackage{hyperref, graphicx}

\numberwithin{equation}{section}

\newtheorem{theorem}{Theorem}[section]
\newtheorem{remark}[theorem]{Remark}
\newtheorem{definition}[theorem]{Definition}
\newtheorem{lemma}[theorem]{Lemma}

\newtheorem{conjecture}[theorem]{Conjecture}
\newtheorem{hypothesis}[theorem]{Hypothesis}
\newtheorem{ansatz}[theorem]{Ansatz}

\begin{document}

\thispagestyle{plain}
\begin{center}
   
    \textbf{\Large Entropy, Thermodynamics and the Geometrization of the Language Model}
        
    \vspace{0.4cm}
    \textbf{Wenzhe Yang}
    
    \vspace{0.1cm}
    \texttt{wenzheyang87@gmail.com}
    
    \vspace{0.1cm}
    
    \url{https://orcid.org/0000-0002-2456-7453}

\end{center}

\begin{abstract}

In this paper, we discuss how pure mathematics and theoretical physics can be applied to the study of language models. Using set theory and analysis, we formulate mathematically rigorous definitions of language models, and introduce the concept of the moduli space of distributions for a language model. We formulate a generalized distributional hypothesis using functional analysis and topology. We define the entropy function associated with a language model and show how it allows us to understand many interesting phenomena in languages. We argue that the zero points of the entropy function and the points where the entropy is close to 0 are the key obstacles for an LLM to approximate an intelligent language model, which explains why good LLMs need billions of parameters. Using the entropy function, we formulate a conjecture about AGI. 

Then, we show how thermodynamics gives us an immediate interpretation to language models. In particular we will define the concepts of partition function, internal energy and free energy for a language model, which offer insights into how language models work. Based on these results, we introduce a general concept of the geometrization of language models and define what is called the Boltzmann manifold. While the current LLMs are the special cases of the Boltzmann manifold.

\vspace{0.3cm}

\noindent \textbf{Keywords}: Language Model, Moduli Space, Entropy Function, Thermodynamics, Geometrization of Language Model, Boltzmann Manifold.

\end{abstract}

\section{Introduction}

In November 2022, OpenAI surprised the world with the release of the highly ``intelligent'' ChatGPT, which possesses many striking abilities that far exceed any other available systems \cite{OpenAI}. For example, it can answer questions with a very high accuracy and engage in human-like conversations. It is the first time that a computer performs well enough in human languages that makes people wonder whether it acquires true intelligence. Ever since then, the arm race in the area of Large Language Model (LLM) has become extremely fierce, with the announcements of new LLMs with hundreds of billions of parameters regularly. Now there are abundant Chatbots based on LLMs available in markets, many of which have astonishing new abilities and behaviors \cite{EmergentQuanta,ChatBot}. Especially, when the number of parameters (e.g., weights) are increased to a threshold, then some surprising behaviors emerges \cite{Emergent1}, which is extremely intriguing. 

An LLM is a generative model, which means it can practically sample a new sentence given a prompt. More precisely, suppose we are given an arbitrary prompt $w_1\cdots w_n$, where each $w_i$ is a word, an LLM can compute the probability distribution of the next word, i.e., $ P(w_{n+1}|w_1 \cdots w_{n})$. Next, we sample the next word $w_{n+1}$ according to this probability distribution \cite{Douglas,Stanford}. After we obtain $w_{n+1}$, we repeat this process and sample $w_{n+2}$ according to the probability distribution $P(w_{n+2}|w_1 \cdots w_{n} w_{n+1})$ computed by the LLM. We repeat this process until some conditions are met, e.g., the LLM outputs an end-of-sequence symbol.

Let us now briefly review the basic structures of the LLMs based on the transformer architecture, to the extent that will be needed in this paper. For a more detailed and thorough treatment, the readers are referred to the book \cite{Stanford}. First, in order for a computer to be able to compute words, words are first embedded into a vector space, which is called the Word-Embedding \cite{MikolovWE}. Suppose $\mathcal{W}$ is the finite set of all words, a Word-Embedding is an injective map
\begin{equation}
    \iota: \mathcal{W} \rightarrow \mathbb{R}^N, \nonumber
\end{equation}
where the standard inner product $\langle \cdot, \cdot \rangle$ on $\mathbb{R}^N$ measures the correlations between words. Given an arbitrary vector $\mathbf{v} \in \mathbb{R}^N$, it defines a distribution on $\mathcal{W}$ via
\begin{equation} \label{eq:BoltzmannD}
    P_{\mathbf{v}}(w) = \frac{\exp \left( \langle \mathbf{v}, \iota(w) \rangle \right)}{\sum_{w' \in \mathcal{W}} \exp \left( \langle \mathbf{v}, \iota(w') \rangle \right)},
\end{equation}
which is called the Boltzmann distribution in statistical mechanics \cite{Muller}.

In particular, given a word $w \in \mathcal{W}$, it also gives us a distribution on $\mathcal{W}$ via
\begin{equation} 
    P_{\iota(w)}(w') = \frac{\exp \left( \langle \iota(w), \iota(w') \rangle \right)}{\sum_{w'' \in \mathcal{W}} \exp \left( \langle \iota(w), \iota(w'') \rangle \right)}, \nonumber
\end{equation}
the value of which measures the correlations between $w$ and $w'$. In practice, the embedding $\iota$ is chosen such that $P_{\iota(w)}(w')$ is the probability that $w'$ is in an $l$-neighborhood of $w$ \cite{Stanford,MikolovWE}. One nice property of this embedding is that generally the embedded vectors $\{ \iota(w)| w \in \mathcal{W} \}$ lie in a lower dimensional vector subspace of $\mathbb{R}^N$ as these vectors are not independent from each other. A famous linear relation is 
\begin{equation}
    \iota(\text{Queen}) - \iota(\text{Woman}) = \iota(\text{King}) - \iota(\text{Man}). \nonumber
\end{equation}
For more details, the readers are referred to the papers \cite{WE1,Douglas,MikolovWE,GloVe}.

After the Word-Embedding, a sentence $w_1\cdots w_n$ is mapped to a vector
\begin{equation}
    (\iota(w_1),\cdots,\iota(w_n)) \in (\mathbb{R}^N)^n. \nonumber 
\end{equation}
As the name suggests, a transformer is a transformation
\begin{equation}
    T_n: (\mathbb{R}^N)^n \rightarrow (\mathbb{R}^N)^n, \nonumber
\end{equation}
which is a differentiable often smooth map \cite{Transformer}. For the details of the construction of a transformer and some of the ideas behind it, the readers are referred to the paper \cite{Transformer} and the book \cite{Stanford}. In particular, each transformer depends on internal parameters which are tuned during the training process. Notice that transformers can be nested in the form
\begin{equation}
    \mathcal{T}_n = T_{n,k}\circ \cdots \circ T_{n,1}, \nonumber 
\end{equation}
where $k$ is the number of layers and different layers have different parameters. Let $\text{Proj}$ be the projection of $(\mathbb{R}^N)^n$ to the last $\mathbb{R}^N$ in the direct product, hence $\text{Proj} \circ \mathcal{T}(\iota(w_1),\cdots,\iota(w_n))$ is a vector of $\mathbb{R}^N$, which defines a distribution on $\mathcal{W}$ via Eq. \eqref{eq:BoltzmannD}. The training of the transformers on a very large corpse is to make sure that this distribution is a good approximation of $P(\cdot|w_1\cdots w_n)$ \cite{Stanford}.

The LLMs have made rapid advance from a numerical and experimental point of view, but there have been very few papers that attempt to understand the properties and behaviors of LLMs from a analytical point of view. More precisely, we now desperately need a theory that can explain these surprising properties and behaviors, which is the main motivation of this paper.

In this paper, we first use set theory, analysis and probability theory to rigorously define the concept of Causal Language Model (CLM), and show it is equivalent to the Predicative Language Model (PLM), which basically predicts the next word of a prompt. We formulate a generalized distributional hypothesis for a CLM based on functional analysis and topology. We introduce the concept of the moduli space of distributions on words, which is the key metric space to understand the power and properties of a language model.

The development of human language is for human to exchange information efficiently between individuals. In order to analytically study a language model, we must understand how information is quantified during conversations. In this paper, we introduce the entropy function for CLM and PLM, and show how it can play a crucial role in the study of language models. In particular, we analyze the zero points of the entropy function, and show how to use information theory to understand the meaning of them. We argue that the zero points of the entropy function and the points where the entropy is close to 0 are the key obstacles for an LLM to approximate an intelligent language model, which explains why good LLMs need billions of parameters. We also formulate an extension conjecture about the properties of AGI using the entropy function.

Next, we use thermodynamics to give an physical interpretation of the language models. In this interpretation, a sentence is a microstate that has a potential energy. All the possible outputs of a prompt define a statistical ensemble, which is distributed according to the Boltzmann distribution \cite{Muller}. We introduce the concept of the partition function, internal energy and Helmholtz free energy for this statistical ensemble associated. Using the Helmholtz free energy, we show how the game of word-predicting can be interpreted as the process of a molecule growing in physics.

Later, we formulate the concept of the geometrization of a language model, and show that the accuracy of a language model is determined by how precise the embedding of the moduli space in a geometrization is. We also define the Word-Embedding with respect to a general manifold with a pairing structure, and introduce the corresponding Boltzmann manifold with respect to it. At last, we show how the current LLMs can be viewed as special cases of the geometrization of language models that use the linear space together with its inner product. Based on our theory, we also pose several important open questions, e.g., which manifold together with a pairing is optimal for the geometrization of language models. Perhaps different languages would need different manifolds and pairings, namely different Boltzmann manifolds in the geometrization. Another important direction is whether the current researches on statistical mechanics and differential geometry can offer new tools in the study of LLMs.

The layout of this paper is as follows. In Section \ref{sec:GLM}, we use set theory, analysis and probability theory to rigorously define language models, i.e., CLM and PLM, and we introduce the moduli space of distributions associated with a language model. In Section \ref{sec:Entropy}, we introduce the entropy function, and show how information theory can be applied to the study of language models. In Section \ref{sec:Thermo}, we present a thermodynamic interpretation of language models and define statistical functions such as partition function, internal energy and Helmholtz free energy for a language model. In Section \ref{sec:GeoLM}, we formulate the concept of the geometrization of language models, and define the Boltzmann manifold. In Section \ref{sec:LLM}, we show that the current LLMs are special cases of the geometrization of language models. In Section \ref{sec:conclusion}, we conclude this paper and propose several open problems.

\section{The Language Model and Moduli Space} \label{sec:GLM}

In this section, we will use set theory, analysis and probability theory to give rigorous definitions of language models. We first define what is a general sentence in the sense of set theory \cite{Logic,Manin,MilneGT}. Then we introduce the definitions of the General Language Model (GLM), the Causal Language Model (CLM) and the Predicative Language Model (PLM). At last, we define what is called the moduli space of distributions for a PLM, which is the crucial metric space to understand the properties and behaviors of language models.

\subsection{Sentence as a Free Mathematical Sequence?}

Suppose $\mathcal{W}$ is the finite set of all words in a language, the set of all sentences associated with $\mathcal{W}$ is defined in the same way as in set theory and logic \cite{Logic,Manin,MilneGT}. First, let $\emptyset$ be the unique sentence with no word, whose length is by definition 0. Let $\mathcal{S}_0$ be the set that consists of a unique element $\emptyset$
\begin{equation}
    \mathcal{S}_0 = \{ \emptyset \}. \nonumber
\end{equation}
\begin{remark}
The inclusion of $\emptyset$ as a sentence will make our analysis simpler from a mathematical point of view. While later we will also see $\emptyset$ is a suitable symbol for the end (or even the beginning) of a sentence.
\end{remark}
A sentence of length 1 is just a word of $\mathcal{W}$. Let $\mathcal{S}_1$ be the set of all sentences of length 1, namely
\begin{equation}
    \mathcal{S}_1 = \mathcal{W}. \nonumber 
\end{equation}
Given an positive integer $n \geq 2$, let $\mathcal{S}_n$ be the set of all sentences of length $n$, i.e., an element of $\mathcal{S}_n$ is a sequence
\begin{equation}
    \mathbf{s}_n=w_1 w_2 \cdots w_n~\text{with}~w_i \in \mathcal{W}. \nonumber 
\end{equation}
Here, the words in a sentence do not need to be different from each other, e.g., we allow the case where $w_i=w_j$, $i \neq j$. The set of all sentences of finite lengths is by definition the union of all $\mathcal{S}_n$
\begin{equation}
    \mathcal{S} = \bigcup_{n=0}^\infty \mathcal{S}_n. \nonumber
\end{equation}
Notice that here we allow the length of a sentence to be arbitrarily large. 
\begin{remark}
In the definition of $\mathcal{S}$, we have allowed the existences of arbitrary sentences that do not satisfy grammar rules. In this paper, the grammar is instead considered as part of the language model.
\end{remark}

\begin{lemma}
The cardinality of $\mathcal{S}$ is countably infinite, i.e., there exists a bijective map between $\mathcal{S}$ and $\mathbb{N}$.
\end{lemma}
\begin{proof}
Since $\mathcal{W}$ is a finite set, this lemma is an immediate result of elementary set theory, see the book \cite{Halmos} for more details.
\end{proof}

Even in this most general setting, we do not consider the case of sentence of infinite length, i.e., infinite sequence, which is very important from both practical and set-theoretic points of view. Practically, human brain has not evolved to understand infinity properly, so it makes no sense to a mortal that a sentence is infinite. Set-theoretically, the set of all infinite sequences has the same cardinality as the real numbers, i.e., its cardinality is $\aleph_1$ \cite{Halmos}. So we do not really want to walk into the deep water of axiomatic set theory.

There is a natural operation defined on $\mathcal{S}$ that is called concatenation: given a sentence $s_m \in \mathcal{S}_m$ and a sentence $s_n \in \mathcal{S}_n$, their concatenation is a new sentence $s_m s_n$ that lies in $\mathcal{S}_{m+n}$. As a convention, the concatenation between $\emptyset$ and $s_n$ is $s_n$ itself, i.e.,
\begin{equation}
    \emptyset \mathbf{s}_n = \mathbf{s}_n \emptyset = \mathbf{s}_n. \nonumber 
\end{equation}
This concatenation operation is associative, but not commutative, while the empty sentence $\emptyset$ serves as a two-sided identity. So the set $\mathcal{S}$ has a monoid structure \cite{MilneGT}.

\subsection{The Causal Language Model}

Now we are ready to give a more mathematical and general definition of the language model.

\begin{definition} \label{defn:DoubleLM}
Given an arbitrary sentence $s \in \mathcal{S}$, a General Language Model (GLM) $\mathscr{L}_G$ defines a joint probability distribution $P_{\mathscr{L}_G,\mathbf{s}}$ on $\mathcal{S} \times \mathcal{S}$, where the probability $P_{\mathscr{L}_G,\mathbf{s}}(\mathbf{s}_1,\mathbf{s}_2)$ measures the likelihood of the concatenation $\mathbf{s}_1\mathbf{s}\mathbf{s}_2$ among all possible choices. Namely, $P_{\mathscr{L}_G,\mathbf{s}}(\mathbf{s}_1,\mathbf{s}_2)$ is the probability that there is a sentence $\mathbf{s}_1$ before $\mathbf{s}$ and a sentence $\mathbf{s}_2$ after $\mathbf{s}$.
\end{definition}

Now, we consider the case where the sentence $\mathbf{s}_1$ on the left side of $\mathbf{s}$ is $\emptyset$, in which case the probability distribution in Definition \ref{defn:DoubleLM} defines a distribution on $\mathcal{S}$. Hence it motivates the following definition.

\begin{definition}
Given an arbitrary sentence $\mathbf{s} \in \mathcal{S}$, a Causal Language Model (CLM) $\mathscr{L}$ defines a probability distribution on $\mathcal{S}$. Namely, given a prompt $\mathbf{s}$, we have a distribution $P(\cdot|\mathbf{s})$ such that the probability that $\mathscr{L}$ outputs a sentence $\mathbf{s}'$ is $P_{\mathscr{L}}(\mathbf{s}'|\mathbf{s})$.
\end{definition}

In the definition, we allow the prompt $\mathbf{s}$ to be $\emptyset$, in which case the CLM $\mathscr{L}$ defines a probability distribution on $\mathcal{S}$, i.e., $P(\mathbf{s}_{\text{out}}|\emptyset)$. It is called causal because $\mathbf{s}'$ is after $\mathbf{s}$, so there is a causality between $\mathbf{s}$ and $\mathbf{s}'$. From the definition, the input to $\mathscr{L}$ can be of arbitrary length, while the output sentence can also be of arbitrary length. In fact, such a definition of CLM is extremely powerful.

Let $\mathcal{B}(\mathcal{S},\mathbb{R})$ be the set of all functions with finite absolute summation, i.e.,
\begin{equation}
    \mathcal{B}(\mathcal{S},\mathbb{R}) = \left\{f:\mathcal{S} \rightarrow \mathbb{R} \bigg| \sum_{\mathbf{s} \in \mathcal{S}}|f(\mathbf{s})| < \infty \right\}. \nonumber 
\end{equation}
The set $\mathcal{B}(\mathcal{S},\mathbb{R})$ is a metric space with a metric given by
\begin{equation}
    \left \lVert f-g \right \rVert  =  \sum_{\mathbf{s} \in \mathcal{S}} \left | f(\mathbf{s}) - g(\mathbf{s})\right |. \nonumber 
\end{equation}
For more details about this space, the readers are referred to the book \cite{SteinReal}. Clearly for every $\mathbf{s}$, the distribution $P_{\mathscr{L}}(\cdot|\mathbf{s})$ is a non-negative function that lies in $\mathcal{B}(\mathcal{S},\mathbb{R})$, in fact $ \left \lVert f \right \rVert = 1 $. Therefore, there exists a map
\begin{equation} \label{eq:DistributionHypothesis}
    \rho_{\mathscr{L}}: \mathcal{S} \rightarrow \mathcal{B}(\mathcal{S},\mathbb{R}).
\end{equation}

In 1950s, linguists formulated the distributional hypothesis, which basically says that the meaning of words are determined by their contexts \cite{Firth,Harris,Joos}. We generalize this hypothesis using real analysis.
\begin{conjecture}
If $\mathscr{L}$ is an intelligent language model, then the map $\rho_{\mathscr{L}}$ is an embedding that is also discrete. Namely, each point $\rho_{\mathscr{L}}(\mathbf{s})$ has a neighborhood that does not contain any other $\rho_{\mathscr{L}}(\mathbf{s}')$ where $\mathbf{s} \neq \mathbf{s}'$. 
\end{conjecture}
For more details of the meaning of ``discrete'' in the sense of topology, the readers are referred to the book \cite{Topology}. As a result, the embedding $\rho_{\mathscr{L}}$ induces a metric on $\mathcal{S}$.

\begin{remark}
It should be stressed that in the definition of a CLM, we do not require it to have ``intelligence'' behaviors. For example, for a given prompt $\mathbf{s}$, the probability distribution $P(\cdot | \mathbf{s})$ can be a ``random'' distribution over $\mathcal{S}$, which can even outputs sentences that are quite pathological. Hence, intelligence is not incorporated in the definition of the language model. Instead, only highly special language models can exhibit intelligence behaviors.
\end{remark}

\subsection{The Predicative Language Model} \label{sec:PLM}

Now we give the definition of the Predicative Language Model.

\begin{definition}
Given a sentence $\mathbf{s}$, a Predicative Language Model (PLM) $\mathscr{L}_P$ defines a probability distribution $P(\cdot|\mathbf{s})$ on $\mathcal{W} \cup \emptyset$. Here, $P(\emptyset|\mathbf{s})$ is the probability that $\mathscr{L}_P$ outputs nothing and $P(w|\mathbf{s})$ is the probability that $\mathscr{L}_P$ outputs the word $w \in \mathcal{W}$.
\end{definition}

We have the following lemma.
\begin{lemma}
The definition of the PLM is equivalent to the definition of the CLM.
\end{lemma}
\begin{proof}
If we have a PLM $\mathscr{L}_P$, then by repetitively using it we can generate sentences. More precisely, given an input $\mathbf{s} \in \mathcal{S}$, we can generate a sentence $\mathbf{s}'=w_1w_2\cdots w_n$ until $\mathscr{L}_P$ outputs $\emptyset$ via
\begin{equation}
    \mathscr{L}_P(\mathbf{s}) = w_1, ~\cdots, ~ \mathscr{L}_P(\mathbf{s}w_1 \cdots w_{n-1}) = w_n, ~\mathscr{L}_P(\mathbf{s}w_1 \cdots w_{n}) = \emptyset. \nonumber 
\end{equation}
The probability that $\mathscr{L}_P$ outputs $\mathbf{s}'=w_1\cdots w_n$ given $\mathbf{s}$ is 
\begin{equation}
    P(\mathbf{s}'|\mathbf{s})=P(w_1|\mathbf{s}) \cdot P(w_2|\mathbf{s}w_1) \cdot  \ldots \cdot P(w_n|\mathbf{s}w_1 \cdots w_{n-1}) \cdot  P(\emptyset|\mathbf{s}\mathbf{s}'). \nonumber 
\end{equation}
In this way, the PLM $\mathscr{L}_P$ defines a probability distribution on $\mathcal{S}$ for any prompt $\mathbf{s}$, hence it produces a CLM.

We now show the inverse construction from CLM to PLM. Namely, suppose we have a CLM $\mathscr{L}$, we now construct a PLM $\mathscr{L}_P$ that reverses the process in the previous paragraph. Given a prompt $\mathbf{s}$, let us construct a distribution on $\mathcal{W} \cup \emptyset$. First, $P(\emptyset|\mathbf{s})$ is straightforward to define, i.e., $P_{\mathscr{L}}(\emptyset|\mathbf{s})$. For an arbitrary word $w$, let \begin{equation}
    P(w|\mathbf{s}) = \sum_{\mathbf{s}' \in \mathcal{S}} P_{\mathscr{L}}(w\mathbf{s}'|\mathbf{s}). \nonumber 
\end{equation}
Notice that here the sum is just over all sentences whose first word is $w$. In this way, we obtain a PLM.
\end{proof}

As shown in the proof of this lemma, the empty sentence $\emptyset$, which is the unit of the monoid $\mathcal{S}$, plays the role of ending the sentence outputted by $\mathscr{L}_P$. Hence comes a crucial definition in our treatment of PLM.

\begin{definition} \label{defn:ModuliSpace1}
The space of all possible distributions on $\mathcal{W} \cup \{\emptyset \}$ that come from a PLM $\mathscr{L}_P$ is called the moduli space of distributions of $\mathscr{L}_P$, which is expressed as
\begin{equation}
    \mathscr{M} = \left\{ P(\cdot| \mathbf{s})| \mathbf{s} \in \mathcal{S} \right\}. \nonumber 
\end{equation}
\end{definition}
In fact, this moduli space is a metric space. Let $\mathcal{B}(\mathcal{W} \cup \{\emptyset\},\mathbb{R})$ be the set of all functions with finite absolute summation
\begin{equation} \label{eq:Bspace}
    \mathcal{B}(\mathcal{W} \cup \{\emptyset \},\mathbb{R}) = \left\{f:\mathcal{W} \cup \{ \emptyset \} \rightarrow \mathbb{R} \bigg| \sum_{w \in \mathcal{W} \cup \{\emptyset \}}|f(w)| < \infty \right\}. \nonumber 
\end{equation}
The set $\mathcal{B}(\mathcal{W} \cup \{\emptyset \},\mathbb{R})$ is a metric space with a metric given by
\begin{equation}
    \left \lVert f-g \right \rVert  =  \sum_{w\in \mathcal{W} \cup \{\emptyset \}} \left | f(w) - g(w)\right |. \nonumber 
\end{equation}
For more details, the readers are referred to the book \cite{SteinReal}. Clearly, for every $\mathbf{s} \in \mathcal{S}$, the distribution $P(\cdot|\mathbf{s}) \in \mathscr{M}$ is a non-negative function that lies in $\mathcal{B}(\mathcal{W} \cup \{\emptyset \},\mathbb{R})$, in fact $ \left \lVert P(\cdot|\mathbf{s}) \right \rVert = 1 $. Therefore, $\left \lVert \cdot \right \rVert $ induces a metric $d$ on the moduli space $\mathscr{M}$
\begin{equation} \label{eq:metricModuli}
    d: \mathscr{M} \times \mathscr{M} \rightarrow \mathbb{R}.
\end{equation}
\begin{remark}
The moduli space $\mathscr{M}$ together with the metric $d$ is the key object to understand a PLM.
\end{remark}

\section{The Entropy Function} \label{sec:Entropy}

In this section, we introduce the concept of the entropy function and show how it quantifies the information contained in the output for a given prompt. We will study the zero-points, which are also the minimal points of the entropy function, and discuss their crucial properties. We argue that these zero points and the points where the entropy is close to 0 are the key obstacles for a LLM to approximate an intelligent language model. At last, we formulate a conjecture regarding AGI.

\subsection{Information Flow and Entropy}

The entropy function for a CLM $\mathscr{L}$ is defined just as in statistical mechanics and information theory \cite{Information,Muller}.

\begin{definition}
Given a prompt $s \in \mathcal{S}$, the entropy of $\mathbf{s}$ with respect to a CLM $\mathscr{L}$ is given by
\begin{equation}
    S_{\mathscr{L}}(\mathbf{s})= - \sum_{\mathbf{s}' \in \mathcal{S}} P_{\mathscr{L}}(\mathbf{s}'|\mathbf{s}) \log P_{\mathscr{L}}(\mathbf{s}'|\mathbf{s}), \nonumber 
\end{equation}
where the sum is over all possible outputs $\mathbf{s}'$.
\end{definition}
We will see that the entropy function is a key mathematical concept to analyze the behaviors of language models. From the definition, the entropy function is a non-negative function defined on $\mathcal{S}$
\begin{equation}
    S_{\mathscr{L}}: \mathcal{S} \rightarrow \mathbb{R}. \nonumber 
\end{equation}
As a special example, if the prompt is $\emptyset$, i.e., no prompt, the entropy associated with it has a special name which is stated as a definition.
\begin{definition}
The background entropy of a CLM $\mathscr{L}$ is
\begin{equation}
    S_{\mathscr{L}}(\emptyset) = - \sum_{\mathbf{s}' \in \mathcal{S}} P_{\mathscr{L}}(\mathbf{s}'|\emptyset) \log P_{\mathscr{L}}(\mathbf{s}'|\emptyset).  \nonumber 
\end{equation}
\end{definition}
The background entropy $S_{\mathscr{L}}(\emptyset)$ is a quantity that gives a zeroth order measurement about how complicated a language is. If $S_{\mathscr{L}}(\emptyset)$ is very large, this intuitively mean that the CLM is quite diverse. While if $S_{\mathscr{L}}(\emptyset)$ is very small, it intuitively means that the CLM is very rigid.

Since the development of language is for individuals to exchange information, so a language model itself must be able to demonstrate this. Here we argue that the definition of the entropy function exactly quantify the information during this process \cite{Information}. More precisely, given a prompt $\mathbf{s}$, the entropy $S_{\mathscr{L}}(\mathbf{s})$ measures how vague are the answers to the prompt $\mathbf{s}$ with respect to this CLM $\mathscr{L}$. If $S(\mathbf{s})$ is very large, this means that the answers can be very diverse. On the other hand, if $S(\mathbf{s})$ is very small, this means the answer to $\mathbf{s}$ is very rigid. Therefore, after the CLM outputs $\mathbf{s}'$, we can say the information revealed by $\mathbf{s}'$ is $S_{\mathscr{L}}(\mathbf{s})$ \cite{Information}.

\begin{remark}
Just like statistical mechanics, we argue that the most important function for a CLM is the entropy function \cite{Muller}.
\end{remark}

\subsection{The Zero Points of the Entropy Function}

Let us now look at the zero points of $S_{\mathscr{L}}$. If $S_{\mathscr{L}}(\mathbf{s})=0$, then it means that the answer to the prompt $\mathbf{s}$ is unique in the CLM $\mathscr{L}$, i.e., there exists a unique sentence $\mathbf{s}' \in \mathcal{S}$ such that $P_{\mathscr{L}}(\mathbf{s}'|\mathbf{s}) = 1$. Hence there is the interpretation that the CLM $\mathscr{L}$ detects that the answer to the prompt $\mathbf{s}$ is unique, and thus the CLM itself actually contain this information. For example, suppose the prompt $\mathbf{s}$ is
\begin{equation}
\mathbf{s} = \texttt{Is Carl Friedrich Gauss a mathematician?} \nonumber 
\end{equation}
If the entropy function of an intelligent CLM $\mathscr{L}$ vanishes at $\mathbf{s}$, i.e., $S_{\mathscr{L}}(\mathbf{s}) = 0$, then this means there is only one possible output of $\mathscr{L}$ for the prompt $\mathbf{s}$. Of course, an intelligent $\mathscr{L}$ that knows about human history will output 
\begin{equation}
    \mathbf{s}' = \texttt{Yes} \nonumber 
\end{equation}
with probability $1$. It should be noticed that at the point $\mathbf{s}$, the vanishing of the entropy only implies there is only one possible output to $\mathbf{s}$. Another CLM $\mathscr{L}_1$ could still have $S_{\mathscr{L}_1}(\mathbf{s})=0$, but it instead outputs 
\begin{equation}
    \mathbf{s}'' = \texttt{No}. \nonumber 
\end{equation}
We say the CLM $\mathscr{L}_1$ outputs a `wrong' answer because Gauss is a mathematician, so the information given by $\mathscr{L}_1$ is wrong. But `wrong' information is still information, the entropy function itself only quantifies the amount of information, but it cannot detect whether the information contained is right or wrong, which needs to be checked with what happens in the outside world. 
\begin{definition}
Given a CLM $\mathscr{L}$, a zero point $\mathbf{s}$ of the entropy function $S_{\mathscr{L}}$ is called a singularity of $\mathscr{L}$.
\end{definition}
Hence the conclusion is that a CLM memorize information at its singularities. We argue that the existence of these singularities and the points where the entropy function is very small are the key obstacle for an LLM to approximate an intelligent CLM.

\subsection{An AGI Conjecture}

It is hard to define which CLM is Artificial General Intelligence (AGI) mathematically! Our idea is that if a CLM $\mathscr{L}$ is an AGI, then it must satisfy certain mathematical properties. In this paper, we formulate a conjecture about this.


\begin{conjecture} \label{conj:AGI}
If a CLM is an AGI, then there exists a finite subset $\mathcal{S}_{\text{AGI}} \subset \mathcal{S}$ such that the entropy function $S_\mathscr{L}: \mathcal{S} \rightarrow \mathbb{R}$ is determined by its value on $\mathcal{S}_{\text{AGI}}$.
\end{conjecture}
The motivation of this conjecture comes from modern mathematics. In modern mathematics, we first give a finite set of definitions and axioms, and then mathematics is about proof, which uses these finite set of rules of prove a result that is not known a priori. The key motivation for \textbf{Conjecture} \ref{conj:AGI} is that the value of $S_\mathscr{L}$ on $\mathcal{S}$ is determined using mathematical deductions by its value on the finite set $\mathcal{S}_{\text{AGI}}$ \cite{Manin}.

\section{A Thermodynamic Interpretation of Language Models} \label{sec:Thermo}

In this section, we will give a thermodynamic interpretation of the CLM defined in Section \ref{sec:GLM}. We show how the concepts and properties of the CLM naturally admit physical interpretations. Then we will define thermodynamic functions such as internal energy, Helmholtz free energy on CLM, and show how they help us understand phenomena in language models. At last, we give a physical interpretation to how a sentence ``grows'' in a PLM.

\subsection{The Boltzmann Distribution}

First, let us introduce the Boltzman distribution in statistical mechanics \cite{Muller}. Suppose we are given a physical system that can occupy different microstates, and suppose the energy of the microstate $i$ is $E_i$. Let $k$ be the Boltzmann constant and $T$ be the temperature, then the probability that the system occupies  microstate $i$ satisfies
\begin{equation}
    p_i \propto \exp \left( - \frac{E_i}{kT} \right). \nonumber 
\end{equation}
Let the partition function of the system be
\begin{equation}
    \mathcal{Z}= \sum_{i} \exp \left( - \frac{E_i}{kT} \right), \nonumber 
\end{equation}
where the sum is over all possible microstates. Then $p_i$ is just
\begin{equation}
    p_i = \frac{1}{\mathcal{Z}} \cdot \exp \left( - \frac{E_i}{kT} \right), \nonumber 
\end{equation}
which is called the Boltzman distribution. The Boltzmann distribution is the distribution that maximizes the entropy
\begin{equation}
    S = - \sum_i p_i \log p_i \nonumber 
\end{equation}
subject to the constraint $\sum p_i =1$ and $\sum_i p_i E_i$ is a constant, i.e., the average energy is kept fixed.

\subsection{A Thermodynamic Interpretation} \label{sec:ThermoInter}

We now give a thermodynamic interpretation of CLM. 
\begin{hypothesis}
In a CLM $\mathscr{L}$, a sentence is a microstate and for every sentence $\mathbf{s} \in \mathcal{S}$, there is an energy $E_{\mathscr{L}}(\mathbf{s})$ associated with it.
\end{hypothesis}
The energy of $\emptyset$ is called the vacuum energy, and the energy of a word $w \in \mathcal{W}$ is called its self-energy \cite{Peskin}.

\begin{remark}
As in physics, a word $w \in \mathcal{W}$ itself can have self-interaction, it has an internal energy $E(w)$, the meaning of this internal energy is that given a random sentence $\mathbf{s}$, the probability that the first word of $\mathbf{s}$ is $w$ is proportional to $\exp \left(-\beta E(w) \right)$. Here $\beta$ is the inverse of temperature.
\end{remark}

Given a length-2 sentence $w_1w_2$, its energy $E_{\mathscr{L}}(w_1w_2)$ is the interaction between the two words $w_1$ and $w_2$, and $E_{\mathscr{L}}(w_1w_2)$ is lower if the interaction between $w_1$ and $w_2$ is stronger. 
\begin{remark}
Since a language is essentially a discrete object, a word does not have a continuous degree of freedom. Therefore, there is no kinetic energy for a word, and only potential energy exists when words form a sentence. 
\end{remark}
Here comes the definition of the statistical ensemble in a language model.
\begin{definition}
For a given prompt $\mathbf{s}$, the statistical ensemble associated with it is 
\begin{equation}
    \left\{ \mathbf{s}\mathbf{s}'|\mathbf{s}' \in \mathcal{S} \right\}. \nonumber 
\end{equation}
\end{definition}

For later purpose, we introduce another quantity $\beta$ which is the inverse of temperature \cite{Muller}. For a fixed CLM, the value of $\beta$ is usually fixed to be 1. 
\begin{hypothesis}
In a CLM $\mathscr{L}$, for a given prompt $\mathbf{s}$ the probability that $\mathscr{L}$ outputs $\mathbf{s}'$ satisfies
\begin{equation}
    P_{\mathscr{L}}(\mathbf{s}'|\mathbf{s}) \propto \exp \left(-\beta E_{\mathscr{L}}(\mathbf{s}\mathbf{s}') \right), \nonumber 
\end{equation}
where $E_{\mathscr{L}}(\mathbf{s}\mathbf{s}')$ is the energy of the concatenation $\mathbf{s}\mathbf{s}'$.   
\end{hypothesis}

More precisely, for a given prompt $\mathbf{s}$, the probability that $\mathscr{L}$ outputs $\mathbf{s}'$ is
\begin{equation}
    P_{\mathscr{L}}(\mathbf{s}'|\mathbf{s}) = \frac{1}{\mathcal{Z}_{\mathscr{L}}} \cdot \exp \left( - \beta E_{\mathscr{L}}(\mathbf{s}\mathbf{s}') \right), \nonumber 
\end{equation}
where $E_{\mathscr{L}}(\mathbf{s}\mathbf{s}')$ is the energy of the concatenation $\mathbf{s}\mathbf{s}'$ and $\mathcal{Z}_{\mathscr{L}}$ is the partition function associated with $\mathbf{s}$
\begin{equation}
    \mathcal{Z}_{\mathscr{L}}(\mathbf{s}) = \sum_{\mathbf{s}' \in \mathcal{S}}  \exp \left( - \beta E_{\mathscr{L}}(\mathbf{s}\mathbf{s}')\right). \nonumber 
\end{equation}
The partition function is in fact a thermodynamic function defined on $\mathcal{S}$
\begin{equation}
    \mathcal{Z}_{\mathscr{L}}:\mathcal{S} \rightarrow \mathbb{R}. \nonumber 
\end{equation}

The average energy of the statistical ensemble associated with $\mathbf{s}$ is 
\begin{equation}
    U_{\mathscr{L}} (\mathbf{s}) = \sum_{\mathbf{s}' \in \mathcal{S}} P_{\mathscr{L}}(\mathbf{s}'|\mathbf{s}) E_{\mathscr{L}}(\mathbf{s}\mathbf{s}'), \nonumber 
\end{equation}
which is called the internal energy of the ensemble. The internal energy $U_{\mathscr{L}}$ is also a well-defined function on $\mathcal{S}$
\begin{equation}
    U_{\mathscr{L}} : \mathcal{S} \rightarrow \mathbb{R}. \nonumber 
\end{equation}
But notice that $U_{\mathscr{L}}$ is different from $E_{\mathscr{L}}$: $E_{\mathscr{L}}(\mathbf{s})$ is the energy of the microstate $\mathbf{s}$, while $U_{\mathscr{L}}(\mathbf{s})$ is the average energy of the statistical ensemble associated with $\mathbf{s}$. Recall that the entropy for the statistical ensemble associated to $\mathbf{s}$ is
\begin{equation}
    S_{\mathscr{L}}(\mathbf{s})=-  \sum_{\mathbf{s}' \in \mathcal{S}} P_{\mathscr{L}}(\mathbf{s}'|\mathbf{s}) \log  P_{\mathscr{L}}(\mathbf{s}'|\mathbf{s}) , \nonumber 
\end{equation}
which is simplified to
\begin{equation}
    S_{\mathscr{L}} = \beta U_{\mathscr{L}} + \log \mathcal{Z}_{\mathscr{L}}. \nonumber 
\end{equation}
The Helmholtz free energy is given by
\begin{equation} \label{eq:Helmholtz}
    F_{\mathscr{L}} = U_{\mathscr{L}} - \frac{1}{\beta} S_{\mathscr{L}}=- \frac{1}{\beta} \log \mathcal{Z}_{\mathscr{L}}. \nonumber 
\end{equation}
The Helmholtz free energy determines whether a sentence $\mathbf{s}$ is likely to be followed by another sentence or not.

\subsection{A Physical Picture of How a Sentence Grows}

Let us now interpret the process of predicting the next word in a PLM using how molecule grows in statistical mechanics. First, because of the equivalence between CLM and PLM, the interpretation in Section \ref{sec:ThermoInter} admits immediate generalization to PLM. In this physical picture, a word is view as an atom, and a sentence is a sequence of atoms, whose energy depends on the order of these atoms. Suppose we are given a prompt $\mathbf{s} \in \mathcal{S}_n$, which is composed of $n$ atoms. To predict the next word is to add an atom to $\mathbf{s}$. For a word $w \in \mathcal{W} \cup \{ \emptyset \}$, the energy of $\mathbf{s}w$ is $E(\mathbf{s}w)$, and the internal energy is by definition
\begin{equation} \label{eq:InternalEPhysics}
     \sum_{w \in \mathcal{W} \cup \{\emptyset \}} P(w|\mathbf{s}) E(\mathbf{s}w).
\end{equation}
We also have the entropy
\begin{equation} \label{eq:entropyphysics}
   - \sum_{w \in \mathcal{W} \cup \{\emptyset \}} P(w|\mathbf{s}) \log P(w|\mathbf{s}).
\end{equation}
There are two trends that fight each other
\begin{enumerate}
    \item the internal energy (Eq. \eqref{eq:InternalEPhysics}) wants to be as small as possible,
    \item the entropy (Eq. \eqref{eq:entropyphysics}) wants to be as large as possible.
\end{enumerate}
As a result, the Helmholtz free energy (Eq. \eqref{eq:Helmholtz}) is minimized, and the resulting distribution is the Boltzmann distribution. After sampling the next word, say $w_{n+1}$, we can continue this process with the new sentence $\mathbf{s}w_{n+1}$ and sample its next one. This process is highly similar to the process of how a molecule grow in statistical mechanics!

But there remains the question of how to compute the energy of a sentence. A sentence can be very long, hence the computation of its energy is a typical many-body problem. The power of LLM is to approximately compute the interaction energy between the ``atoms'' in a sentence using machine learning.

\section{The Geometrization of Language Models} \label{sec:GeoLM}

In this section, we will introduce the concept of the geometrization of a PLM, which lays the foundation for studying language models using differential geometry. We will also define a more general version of the Word-Embedding, which leads to the concept of a Boltzmann manifold.

\subsection{The Geometrization of PLM}

Suppose we have a PLM $\mathscr{L}_P$, whose moduli space of distributions is $\mathscr{M}$ (Definition \ref{defn:ModuliSpace1}). Moreover, we have a map
\begin{equation}
    \rho: \mathcal{S} \rightarrow \mathscr{M}, \nonumber 
\end{equation}
where $\rho$ sends a sentence $\mathbf{s} \in \mathcal{S}$ to the distribution $P(\cdot| \mathbf{s}) \in \mathscr{M}$. Notice that $\mathscr{M}$ is a metric space, the details of which can be found in Section \ref{sec:PLM}.

Generally, the moduli space $\mathscr{M}$ itself is extremely complicated, to study which we need to borrow ideas from both mathematics and physics. The first step is to geometrize this complicated metric space, which ``prompts'' the following definition.

\begin{definition}
The $\epsilon$-geometrization of a PLM is a manifold $M$ and a continuous map
\begin{equation}
    \Psi: \mathscr{M} \rightarrow M, \nonumber 
\end{equation}
where each point of $M$ represents a probability distribution on $\mathcal{W} \cup \emptyset$ that is also computable. Moreover, for every point $P(\cdot| \mathbf{s})$ of $\mathscr{M}$, the norm of the difference $\left \Vert \Psi(P(\cdot| \mathbf{s}))-P(\cdot| \mathbf{s}) \right \Vert < \epsilon(n)$. Here, the error $\epsilon(n)$ can depend on the length of the sentence $\mathbf{s} \in \mathcal{S}_n$.
\end{definition}
Notice that the norm $\left \Vert \cdot \right \Vert$ is defined in the metric space $\mathcal{B}(\mathcal{W} \cup \{\emptyset \},\mathbb{R})$ (Eq. \eqref{eq:Bspace}). The composition $\Psi \circ \rho$ is a map from $\mathcal{S}$ to $M$
\begin{equation}
     \Psi \circ \rho: \mathcal{S} \rightarrow M. \nonumber 
\end{equation}
The motivation of the geometrization of the moduli space $\mathscr{M}$ is that we try to find a finite dimensional manifold $M$ to approximately describe the infinite set $\mathscr{M}$. Then we use the geometry of this manifold to study the bizarre metric space $\mathscr{M}$.

\subsection{The Word-Embedding and the Boltzmann Manifold} \label{sec:WordEmbedding}

One way to construct an $\epsilon$-geometrization is via an additional structure on a manifold called the pairing. Suppose $M$ is a manifold, a symmetric pairing is a smooth function 
\begin{equation}
    \langle \cdot, \cdot \rangle : M \times M \rightarrow \mathbb{R}, \nonumber
\end{equation}
which maps any two points $\text{pt}_1, \text{pt}_2 \in M$ to a number in $\mathbb{R}$. Here ``symmetric'' means
\begin{equation}
     \langle \text{pt}_1, \text{pt}_2 \rangle = \langle \text{pt}_2, \text{pt}_1 \rangle. \nonumber 
\end{equation}
A Word-Embedding is an injective map
\begin{equation}
    \iota: \mathcal{W} \cup \{\emptyset \} \rightarrow M. \nonumber 
\end{equation}
Given an arbitrary point $\text{pt} \in M$, it defines a distribution $P_{\text{pt}}$ on $\mathcal{W} \cup \{ \emptyset \}$ via
\begin{equation}
    P_{\text{pt}}(w) =  \frac{\exp \left( - \beta \langle pt,\iota(w) \rangle \right)}{\sum_{w' \in \mathcal{W} \cup \emptyset} \exp \left( - \beta  \langle pt,\iota(w') \rangle \right)}, \nonumber 
\end{equation}
which is exactly from the Boltzmann distribution \cite{Muller}. Hence every point of $M$ defines a distribution on $\mathcal{W} \cup \{ \emptyset \}$ and $M$ is a space of distributions on $\mathcal{W} \cup \{ \emptyset \}$.

\begin{definition}
The manifold $M$ together with the pairing $\langle \cdot,\cdot \rangle$ and the Word-Embedding $\iota$ is called a Boltzmann manifold. 
\end{definition}

Since every point of the Boltzmann manifold $M$ defines a distribution on $\mathcal{W} \cup \{\emptyset \}$, $M$ is a subspace of  $\mathcal{B}(\mathcal{W} \cup \{\emptyset \},\mathbb{R})$ (Eq. \eqref{eq:Bspace}). Therefore, given two points $\text{pt}$ and $\text{pt}'$ of $M$, their distance is
\begin{equation}
    d(\text{pt},\text{pt}') =   \left \Vert P_{\text{pt}}-P_{\text{pt}'} \right \Vert. \nonumber
\end{equation}
This distance function induces a metric $g_B$ on $M$, which is called the Boltzmann metric, hence $M$ is in fact a Riemannian manifold \cite{LeeR}.

Now suppose $M$ gives us an $\epsilon$-geometrization with the Moduli-Embedding $\Psi: \mathscr{M} \rightarrow M$. Since $\mathcal{S}_n$ is the direct product $\mathcal{W}^n$, thus if we believe the map $ \Psi \circ \rho $ has sufficiently nice properties, then there exists a diagram
\begin{equation}
\begin{CD}
\mathcal{S}_n = \mathcal{W}^n @>{\iota^n}>> M^n\\
@VV{ \rho}V @VV{\Lambda_n}V\\
\mathscr{M} @>\Psi>> M
\end{CD}~~, \nonumber 
\end{equation}
which is commutative up to a small error that depends on $n$. More precisely, we assume there exists a smooth map $\Lambda_n:M^n \rightarrow M$ such that for every $\mathbf{s}=w_1w_2\cdots w_n$, we have
\begin{equation}
    \Psi \circ \rho (\mathbf{s}) = \Lambda_n\left( \iota(w_1), \cdots, \iota(w_n) \right). \nonumber 
\end{equation}
It means that the map $\rho \circ \Psi$ is determined by the map $\iota$ and $\left\{\Lambda_n:n \geq 2 \right\}$. The point is that the map $\Psi \circ \rho $ is defined on the infinite set $\mathcal{S}$, which can be highly complicated. While the map $\iota$ is defined on a finite set and $\Lambda_n$ is defined on a smooth manifold, hence can be studied by the modern theory of differential geometry \cite{Lee}.

\section{LLM as a Special Geometrization} \label{sec:LLM}

In this section, we show that the current LLMs based on the transformer architecture are special cases of geometrization in the sense of Section \ref{sec:GeoLM}. We also introduce the fundamental thermodynamic relation of LLMs.

\subsection{The Geometrization and LLM}

The Boltzmann manifold used in LLMs is the simplest one: $\mathbb{R}^N$ together with the standard inner product $\langle \cdot , \cdot \rangle $. In practice, first we obtain a Word-Embedding  \cite{MikolovWE}
\begin{equation}
    \iota: \mathcal{W} \cup \{ \emptyset \} \rightarrow \mathbb{R}^N. \nonumber 
\end{equation}
Then a point $\mathbf{v}$ of $\mathbb{R}^N$ defines a distribution on $\mathcal{W} \cup \{ \emptyset \}$ via 
\begin{equation} \label{eq:VectorDistribution}
    P_{\mathbf{v}}(w) = \frac{\exp \left(-\beta \langle \mathbf{v}, \iota(w)\rangle \right)}{\sum_{w' \in \mathcal{W} \cup \{ \emptyset \} } \exp \left(-\beta \langle \mathbf{v}, \iota(w')\rangle \right)}, 
\end{equation}
where $\beta$ is the inverse of temperature. The Moduli-Embedding is a map
\begin{equation}
\Phi: \mathscr{M} \rightarrow \mathbb{R}^N. \nonumber 
\end{equation}
Given a sentence $\mathbf{s} \in \mathcal{S}$, $\Psi \circ \rho(\mathbf{s})$ is a vector of $\mathbb{R}^N$, which gives the distribution of the next word of $\mathbf{s}$ via Eq. \eqref{eq:VectorDistribution}, i.e.,
\begin{equation}
    P(w|\mathbf{s}) = \frac{\exp \left(-\beta \langle \Psi \circ \rho(\mathbf{s}), \iota(w)\rangle \right)}{\sum_{w' \in \mathcal{W} \cup \{ \emptyset \} } \exp \left(-\beta \langle \Psi \circ \rho(\mathbf{s}), \iota(w')\rangle \right)}.\nonumber 
\end{equation}

\begin{ansatz}
For every $n \geq 1$, there exists a map $\Lambda_n: (\mathbb{R}^N)^n \rightarrow \mathbb{R}^N$ that makes the following diagram commutative up to a small error
\begin{equation}
\begin{CD}
\mathcal{S}_n = \mathcal{W}^n @>{\iota^n}>> (\mathbb{R}^N)^n\\
@VV{\rho}V @VV{\Lambda_n}V\\
\mathscr{M} @>\Psi>> \mathbb{R}^N
\end{CD}. \nonumber 
\end{equation}
\end{ansatz}
Namely, given a sentence $\mathbf{s}=w_1\cdots w_n$, the distribution of the word after $\mathbf{s}$ is given by the vector $\Lambda_n(\iota(w_1),\cdots,\iota(w_n))$ via Eq. \eqref{eq:VectorDistribution}. In practice, both $\iota$ and $\Lambda_n$ are determined by data and machine learning \cite{Stanford}. Currently, the most popular method to construct the map $\Lambda_n$ is based on the transformer architecture \cite{Stanford,Transformer}.

\subsection{The Fundamental Thermodynamic Relation of LLMs}

Given a geometrization $\left(\mathbb{R}^N, \langle\cdot,\cdot \rangle, \iota \right)$, the properties of $w \in \mathcal{W} \cup \{ \emptyset \}$ with respect to this geometrization are determined by the embedding vector $\phi(w) \in \mathbb{R}^N$. Borrowing language from physics \cite{Landau}, the space $\mathbb{R}^N$ is also called the phase space of a language model. Just as in CLM, in a PLM we also have the entropy function $S$, the internal energy $U$ and Helmholtz free energy $F$ defined on $\mathcal{S}$.

\begin{conjecture}
Suppose we are given a geometrization $\left(\mathbb{R}^N, \langle\cdot,\cdot \rangle, \iota \right)$ and the maps $\{ \Lambda_n| n\geq 1 \}$ associated with a PLM. The finite set $\mathcal{S}_n$ is viewed as a subspace of $(\mathbb{R}^N)^n$ via the embedding $\iota^n$. Then all thermodynamic functions such as $S$, $U$ and $F$ defined over $\mathcal{S}_n$ admit smooth extensions to functions over $(\mathbb{R}^N)^n$ up to small errors.
\end{conjecture}
Let us now look at the entropy function $S$ as an example. This conjecture implies that there exists a differentiable function $\widetilde{S}: (\mathbb{R}^N)^n \rightarrow \mathbb{R}$ that makes the following diagram commutative up to a small error
\begin{equation}
\begin{CD}
\mathcal{S}_n = \mathcal{W}^n @>{\iota^n}>> (\mathbb{R}^N)^n\\
@VV{\rho}V @VV{\widetilde{S}}V\\
\mathscr{M} @>S>> \mathbb{R}
\end{CD}. \nonumber 
\end{equation}

\begin{remark}
In fact, from the construction of LLMs using the architecture of transformer, the extensions of thermodynamic functions always exist, but the error cannot be guaranteed to be small.
\end{remark}

Suppose the coordinate of $(\mathbb{R}^N)^n$ is $x_{i,j}$ with $1 \leq i \leq N$ and $1 \leq j \leq n$, then we have the fundamental thermodynamic relation of LLMs
\begin{equation}
    d\widetilde{U} = \frac{1}{\beta} d\widetilde{S} -\sum_{i,j}f_{i,j}dx_{i,j}, \nonumber 
\end{equation}
where $f_{i,j}$ is the generalized force given by
\begin{equation}
    f_{i,j}=\partial \widetilde{S}/\partial x_{i,j}. \nonumber 
\end{equation}

As a function on $(\mathbb{R}^N)^n$, we conjecture that the entropy function $\widetilde{S}$ for a LLM is convex. 
\begin{conjecture}
Suppose $0 \leq \lambda  \leq 1$, given two points $\mathbf{x}=(x_{i,j})$ and $\mathbf{x}'=(x'_{i,j})$ of $(\mathbb{R}^N)^n$, we have
\begin{equation}
\widetilde{S}\left(\lambda \mathbf{x} + (1 - \lambda) \mathbf{x}' \right) \geq \lambda \widetilde{S}\left( \mathbf{x} \right) + (1 - \lambda ) \widetilde{S} \left(\mathbf{x}' \right). \nonumber 
\end{equation}
\end{conjecture}
It has an intuitive interpretation, if you mix the two vectors $\mathbf{x}=(x_{i,j})$ and $\mathbf{x}'=(x'_{i,j})$, the entropy is always greater than the average. Namely, if you mixed the meanings, it always becomes more vague.

\section{Conclusion and Open Problems} \label{sec:conclusion}

In this paper, we have used set theory, analysis and probability theory to rigorously give the definitions of Causal Language Model (CLM) and Predicative Language Model (PLM). We propose a generalization of the distributional hypothesis. We introduce the concept of the moduli space of distributions for a CLM, which is the key metric space to understand the power and properties of a language model. Then we define the entropy function for a CLM, study its properties, and show how its zero points are crucial for a language model to incorporate information and logic. Using the entropy function, we formulate an extension conjecture about AGI.

We have used thermodynamics to show that language models admit thermodynamic interpretations. Namely, a sentence is a microstate and a prompt gives us a statistical ensemble, which is distributed according to the Boltzmann distributions. We also formulate the concept of the geometrization of language models and define the Boltzmann manifold. We show the embedding of the moduli space into a Boltzmann manifold determines how accurate can a geometrization be. At last, we discuss how the current LLMs can be viewed as special cases of geometrization with respect to the linear space together with its inner product.

There are many interesting open questions that deserve further study. Perhaps the most straightforward question is what is the best Boltzmann manifold for the geometrization of a language model. Current LLMs all uses the linear space together with its inner product. But it is likely that different languages would need different Boltzmann manifold for the geometrization. For example, some of the possible choices of Boltzmann manifold with non-trivial topology can be
\begin{equation}
    \mathbb{R}^N \times S^M,~ \mathbb{R}^N \times (S^1)^M, \nonumber 
\end{equation}
where $S^k$ is a $k$-dimensional sphere. For such manifold, it is very interesting to see whether the non-trivial topology has any effect on the properties of the resulting LLM.

Another important question is whether results in modern physics and mathematics can be applied to study language models. For example, whether statistical physics can help us understand the emergent abilities of LLMs. Whether the results in statistical mechanics and differential geometry can offer a better architecture than the transformer. There are numerous results in the study of many-body interaction systems, it is interesting to see whether these results can have applications to LLMs.

\bibliographystyle{plain}
\bibliography{refs_SF2}

\end{document}